\documentclass{article}
\usepackage{ijcai23}

\usepackage{times}
\usepackage{soul}
\usepackage{url}
\usepackage[hidelinks]{hyperref}
\usepackage[utf8]{inputenc}
\usepackage[small]{caption}
\usepackage{graphicx}
\usepackage{amsmath}
\usepackage{amsthm}
\usepackage{booktabs}
\usepackage{algorithm}
\usepackage{algorithmic}
\usepackage[switch]{lineno}
\usepackage{amssymb}
\usepackage{amsfonts}
\usepackage{bm}
\usepackage{color}
\usepackage{pifont}
\usepackage{xspace}
\usepackage{multirow}
\usepackage{tabularx}
\usepackage{makecell}
\usepackage{subfig}
\usepackage{arydshln}
\usepackage{threeparttable}
\usepackage{tabularx}

\newtheorem{theorem}{Theorem}
\newtheorem{lemma}{Lemma}

\newcommand{\relu}{ReLU\xspace}
\newcommand{\reluop}{\mathrm{ReLU}}
\newcommand{\iter}{\mathrm{iter}}
\newcommand{\dnn}{DNN\xspace}
\newcommand{\urfornnv} {\texttt{UR4NNV}\xspace}
\newcommand{\sat}{SAT\xspace}
\newcommand{\smt}{SMT\xspace}
\newcommand{\dpll}{DPLL\xspace}
\newcommand{\safe}{\textsc{Sa}\xspace}
\newcommand{\unsafe}{\textsc{Uns}\xspace}
\newcommand{\UNK}{\textsc{Unk}\xspace}

\DeclareMathOperator*{\sigmoid}{sigmoid}

\title{\urfornnv: Neural Network Verification, Under-approximation Reachability Works!}

\author{
Zhen Liang$^1$
\and
Taoran Wu$^{2,3}$\and
Ran Zhao$^{1}$\and
Bai Xue$^{2}$ \and
Ji Wang$^1$ \and
Wenjing Yang$^1$ \and \\
Shaojun Deng$^1$ \And
Wanwei Liu$^{1,4}$ 
\affiliations
$^1$College of Computer Science and Technology, National University of Defense Technology\\
$^2$Institute of Software, Chinese Academy of Sciences\\
$^3$School of Computer Science and Technology, University of Chinese Academy of Sciences\\
$^4$Key Laboratory of Software Engineering for Complex Systems, National University of Defense Technology
\emails
\{liangzhen, zhaoran13, wj, wenjing.yang, dengshaojun, wwliu\}@nudt.edu.cn, \\
\{wutr, xuebai\}@ios.ac.cn
}

\begin{document}

\maketitle

\begin{abstract}

Recently, formal verification of deep neural networks ({\dnn}s) has garnered considerable attention, and over-approximation based methods have become popular due to their effectiveness and efficiency. However, these strategies face challenges in addressing the ``unknown dilemma" concerning whether the exact output region or the introduced approximation error violates the property in question. To address this, this paper introduces the \urfornnv verification framework, which utilizes under-approximation reachability analysis for DNN verification for the first time. \urfornnv focuses on {\dnn}s with Rectified Linear Unit (\relu) activations and employs a binary tree branch-based under-approximation algorithm. In each epoch, \urfornnv under-approximates a sub-polytope of the reachable set and verifies this polytope against the given property. Through a trial-and-error approach, \urfornnv effectively falsifies {\dnn} properties while providing confidence levels when reaching verification epoch bounds and failing falsifying properties. Experimental comparisons with existing verification methods demonstrate the effectiveness and efficiency of \urfornnv, significantly reducing the impact of the ``unknown dilemma".

\end{abstract}

\section{Introduction}
\label{intro}

Deep neural networks ({\dnn}s), also known as NNs, have achieved remarkable success in the field of artificial intelligence in recent decades and have gained significant popularity in a variety  domains, including  natural language processing~\cite{DBLP:conf/nips/YuanNL21,DBLP:conf/nips/KarchTHMO21} and autonomous driving~\cite{DBLP:journals/firai/LiuWLTW23}. However, {\dnn}s are not infallible and often exhibit  unexpected behaviors, which can lead to system failures and even loss of life. Consequently, it is crucial to subject {\dnn}s to formal verification before practical deployment tn ensure their satisfaction of required properties such as safety, robustness, and fairness, particularly in safety-critical domains. 


Recently, numerous verification approaches have been proposed \cite{kochdumper2022open,huang2019reachnn,JCST-2207-12703}. These approaches can be categorized into two classes: \textit{sound and complete verification methods}, and \textit{sound and incomplete verification methods}. 
The former involves an exhaustive examination of whether the output space of DNNs violates certain properties, often utilizing SMT/SAT solvers \cite{10.1007/978-3-319-63387-9_5,ehlers2017formal} or reachability analysis~\cite{DBLP:journals/corr/abs-1712-08163}. Sound and complete verification methods can provide deterministic conclusions, such as \textsc{Sat} or \textsc{Unsat}, regarding the given properties. However, these methods suffer from prohibitive computation burdens and limited scalability, with verification of small DNNs taking hours or days and larger DNNs being practically infeasible. In this paper, sound and complete verification methods are also referred to as \textit{exact verification methods}.


To alleviate the verification duration significantly, sound and incomplete verification methods are introduced. As the name implies, these methods provide deterministic conclusions of \textsc{Sat} or \textsc{Unsat} but may fail to verify some properties, resulting in \textsc{Unknown} conclusions. The \textsc{Unknown} results arise from the tradeoff made by these methods to speed up DNN verification at the expense of verification accuracy. These approaches focus on over-approximating the exact output space and checking the over-approximated output spaces against the given properties \cite{gehr2018ai2,10.5555/3327546.3327739}. Convex optimization or reachability analysis are commonly employed in these methods to over-approximate the exact output regions. However, over-approximation brings forth uncertainties regarding whether the introduced errors or the exact output spaces violate certain properties in \textsc{Unknown} cases, creating an unknown dilemma. Sound and incomplete verification methods are also referred to as \textit{over-approximation verification methods}.


To address the unknown dilemma, this paper proposes a novel verification framework named \textbf{U}nder-approximation \textbf{R}eachability for \textbf{N}eural \textbf{N}etwork \textbf{V}erification ({\urfornnv}). {\urfornnv} leverages under-approximation reachability analysis to distinguish cases where the exact output space or the approximation error violates DNN properties. The exact output region of ReLU DNNs consists of a set of polytopes, and {\urfornnv} adopts a random under-approximation reachability strategy to represent a sub-polytope of the exact output region during each epoch. It examines the under-approximation branch with respect to the given property. If an under-approximation polytope violates the required property, the verification terminates and the property is proven to be unsatisfiable (\textsc{Unsat}). With a large number of under-approximation epochs, {\urfornnv} goes through the exact polytope set to provide deterministic conclusions. If no property violations occur, {\urfornnv} returns satisfiable (\textsc{Sat}) with high confidence.


\textbf{Contribution.} The main contributions of this paper are listed as follows:
\begin{itemize}
    \item This paper focuses on addressing the ``unknown dilemma" of over-approximation verification approaches for \relu\ {\dnn}s and proposes the \urfornnv framework. This framework is the first work that focuses on verifying {\dnn}s with under-approximation and it has the potential to open new research directions in developing novel verification paradigms for {\dnn}s.
    
    
    \item  The \urfornnv framework randomly under-approximates a sub-polytope of the exact output region of \relu\ {\dnn}s and checks it against desired properties during each epoch. It employs an efficient and scalable under-approximation algorithm with a bounded vertex number. Although the \urfornnv framework is particularly recommended for DNN falsification, it can also confirm a property with high confidence under a sufficient number of iterations.
    
    
    \item  Several optimization strategies are developed to improve the efficiency and effectiveness of \urfornnv. These strategies include dimension-priority assignment, pruning strategy, maximal under-approximation, and parallel execution.
    

    \item We have implemented a prototype tool based on the \urfornnv verification framework. The tool was evaluated on widely used ACAS Xu {\dnn} benchmarks, comparing its performance to that of exact and over-approximation verification methods. The experimental results demonstrate a significant improvement in both the effectiveness and efficiency of the proposed tool compared to existing methods.
    
    
\end{itemize}


\section{Preliminaries and Notations}
\label{prel}

Given an $L$-layer \relu\ {\dnn} $\mathcal{N}$, its computation w.r.t. the input $\bm{x}$ can be defined as follows:
\begin{equation}
    \begin{cases}
         \bm{z}^{(l)}= \bm{W}^{(l)} \bm{h}^{(l-1)}+ \bm{b}^{(l)},~ \bm{h}^{(l)}=\sigma( \bm{z}^{(l)})\\
         \bm{h}^{(0)}=\bm{x},~ \mathcal{N}(\bm{x})=\bm{z}^{(L)}\\
         \bm{W}^{{(l)}}\in \mathbb{R}^{n_l\times n_{l-1}},~
         \bm{b}^{(l)} \in \mathbb{R}^{n_l}\\
         l \in \{1,2,\cdots,L-1\},\\
    \end{cases}
\end{equation}
where $\bm{W}^{(l)}$ and $\bm{b}^{(l)}$ respectively represent the weight matrix and bias vector between the $(l-1)$-th layer and the $l$-th layer, forming affine transformations. $n_{l}$ refers to the number of neurons of the $l$-th layer, named layer dimension. Subsequently, the element-wise activation function $\sigma(x)=\reluop(x)=\max(0,x)$  receives the pre-activation value $\bm{z}^{(l)}$ outputted from the affine transformation and generates the activation value $\bm{h}^{(l)}$. The composition of the computation operators involved in all layers is termed the forward propagation, obtaining the final output $\mathcal{N}(\bm{x})$ with respect to the input $\bm{x}$.

{\dnn} verification provides conclusions on whether {\dnn}s satisfy certain properties and reachability analysis is one of the most potent tools for {\dnn} verification. Properties to be verified generally depict two regions: an input region $\mathcal{X}$ and a desired output region $\mathcal{Y}$.  Given an input region $\mathcal{X}$, the output region of a {\dnn} $\mathcal{N}$ is 
$\mathcal{R}(\mathcal{X})=\{\bm{y} \mid \bm{y}=\mathcal{N}(\bm{x}), \ \bm{x}\in \mathcal{X}\}.$ If the output region is contained in the desired output region, i.e., $\mathcal{R}(\mathcal{X})\subseteq\mathcal{Y}$, then the property holds on the {\dnn}. Otherwise, the {\dnn} violates the property. Polytopes are popularly utilized in reachability analysis-based {\dnn} verification.

A polytope $P$ is the convex hull of a finite set of points in Euclidean space $\mathbb{R}^{n}$. Generally speaking, there are two representation methods of polytopes: $\mathcal{H}$-representation and $\mathcal{V}$-representation. With the $\mathcal{H}$-representation, a polytope is defined with finitely many half spaces, described by the form of linear inequalities $P\triangleq\{\bm{x}\mid \bm{Ax} \leq \bm{c},\bm{A}\in \mathbb{R}^{m\times n},\bm{c}\in \mathbb{R}^{m\times 1}\}$, and such that the point set obeys all of them is bounded. For the $\mathcal{V}$-representation, the convex hull is defined with the set of vertices $P\triangleq 
\{\sum_{i=0}^{m}\lambda_{i}\bm{v}_{i}\mid \lambda_{0},\lambda_{1},\cdots, \lambda_{m}>0,\sum_{i=1}^{m}\lambda_i=1\}$. Theoretically, these two representations are equivalent and are interchangeable \cite{avis2009polyhedral}. However, there are no efficient approaches to the representation conversion problem, especially for high-dimensional cases \cite{Bremner2007PolyhedralRC,DBLP:journals/comgeo/WelzlABS97,DBLP:journals/dcg/BremnerFM98,DBLP:journals/mor/MatheissR80}. For a polytope $P$ with dimension $n$, it means that there exist $n+1$ points $\{\bm{v}_0, \bm{v}_1,\ldots,\bm{v}_n\}$ making the set $\{\bm{v}_i-\bm{v_0}\mid i=1,\ldots,n\}$ linearly independent. The following lemmas demonstrate some concerned computation properties of polytopes in \relu\ {\dnn}s. 
\begin{lemma}
\label{poly_aff}
     Affine transformations only change vertices but preserve polytopes' combinatorial structure, i.e., polytopes are closed under affine transformations~\cite{10.5555/285869.285884}.
\end{lemma}

\begin{lemma}
\label{poly_relu}
     For \relu\ activations $y=\max(x,0)$, if the input $x$ is a union of polytopes, then the output $y$ is also a union of polytopes~\cite{DBLP:journals/corr/abs-1712-08163}.
\end{lemma}

 We use lowercase letters like $a$, $b$, $c$ to range over scalars, bold lowercase letters such as $\bm{b}$, $\bm{s}$ and  $\bm{x}$ to range over vectors, and bold uppercase letters to refer to matrices, such as $\bm{M}$, $\bm{V}$ and $\bm{N}$.
For vector $\bm{b}$, $\bm{b}_{i}$ refers to the $(i+1)$-th scalar of $\bm{b}$. For matrix $\bm{M}$, $\bm{M}_{i,:}$ (resp. $\bm{M}_{:,j}$)  stands for the $(i+1)$-th row (resp. the $(j+1)$-th column) of $\bm{M}$ and $\bm{M}_{i,j}$ denotes the element located in the $i+1$-th row and $j+1$-th column. Moreover, we denote the number of the rows of $\bm{M}$ with $|\bm{M}|$.
\section{Methodology}
\label{method}

\subsection{Dilemma and Insight}

\begin{figure}[htbp]
    \centering
    \includegraphics[scale=0.42]{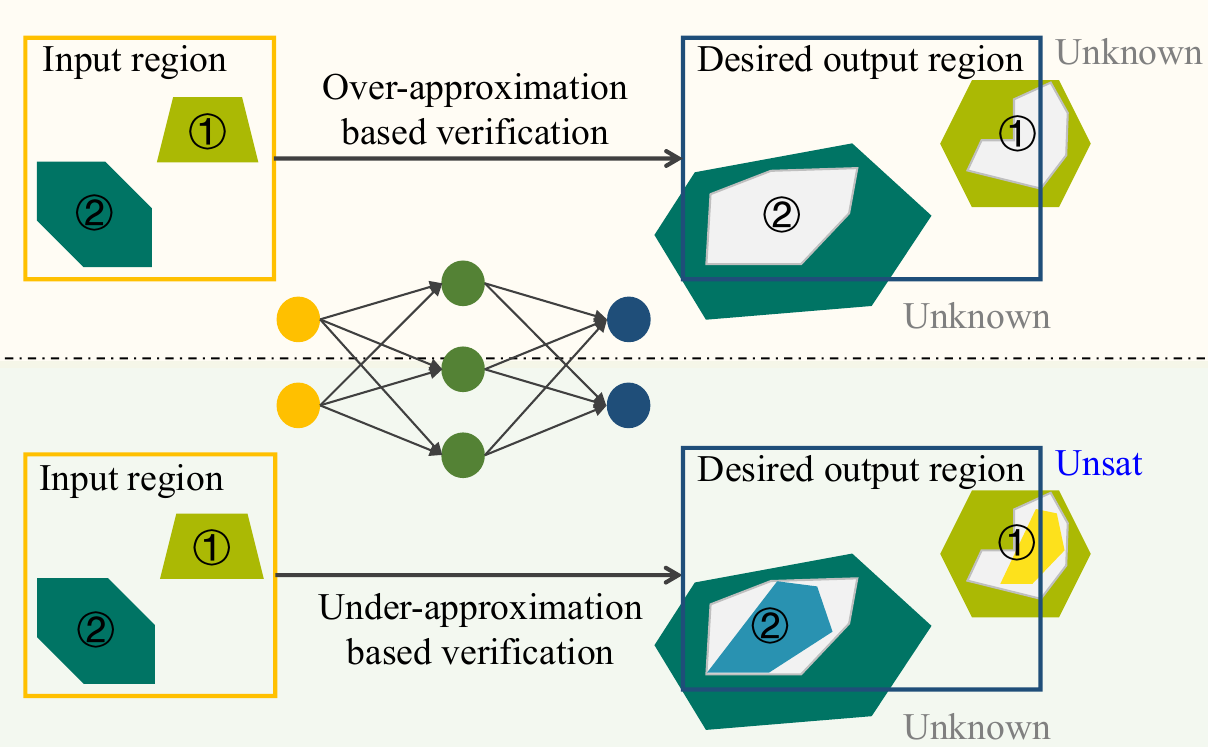}
    \caption{Illustration on the unknown dilemma and coping insight.}
    \label{fig:enter-label}
\end{figure}

Combining Lemma \ref{poly_aff} and \ref{poly_relu}, it can be concluded that the output reachable region of \relu DNNs is a union of polytopes. However, the number of polytopes consisting of the final output reachable region is prohibitive ~\cite{10.5555/2969033.2969153,Serra2017BoundingAC} and the computation of the exact reachable region is NP-hard~\cite{10.1007/978-3-319-63387-9_5}. Consequently, the over-approximation techniques come readily. Given an input region $\mathcal{X}$ and a desired output region $\mathcal{Y}$ with respect to property $p$, over-approximation on \dnn $\mathcal{N}$ aims to find an outer-approximated region $\Omega(\mathcal{X})$ of the exact reachable region $\mathcal{R}(\mathcal{X})$, i.e., $\mathcal{R}(\mathcal{X})\subseteq \Omega(\mathcal{X})$. Then, under the condition that the superset of the output region is contained in the desired output region, property $p$ holds on the\dnn $\mathcal{N}$, i.e., 
\begin{equation}
\label{eq:over}
    \Omega(\mathcal{X})\subset \mathcal{Y} \Rightarrow \mathcal{N}\models p.
\end{equation}

However, over-approximation does not always work. Under the condition that $\Omega(\mathcal{X})\not\subset \mathcal{Y}$, the verification fails in seeking out whether the exact output region $\mathcal{R}(\mathcal{X})$ or the introduced approximation error $\Omega(\mathcal{X})\setminus\mathcal{R}(\mathcal{X})$ violates the desired output region, leading to the \textit{Unknown dilemma}. The top part of Figure~\ref{fig:enter-label} visualizes these claims, where the grey areas are the exact output regions and other colored areas are the over-approximated output regions of the corresponding input regions. It can be seen that the first case actually violates the given property and the second obeys, while the over-approximation based verification cannot distinguish them.

To alleviate this dilemma, we resort to the under-approximation based verification, which essentially concentrates on constructing an under-approximated region $\Theta(\mathcal{X})$ of the output reachable set $\mathcal{R}(\mathcal{X})$, i.e., $\Theta(\mathcal{X})\subseteq \mathcal{R}(\mathcal{X})$. If the subset of the exact reachable region is not completely contained in the desired output region, the {\dnn} violates the property. That is to say, 
\begin{equation}
\label{eq:under}
    \Theta(\mathcal{X})\not \subset \mathcal{Y} \Rightarrow \mathcal{N}\not\models p.
\end{equation}
This can help to identify the cases where the exact output region violates the desired output region to some extent. Moreover, on the premise that the under-approximation approaches the exact output polytope set as close as possible, the under-approximation distinguishes between these two case. These claims are shown in the bottom part of Figure~\ref{fig:enter-label}, where the colored areas located in the exact output regions are under-approximation output sets. In what follows, we demonstrate the under-approximation based verification for \relu\ {\dnn}s.

\subsection{Under-approximation of \relu\ {\dnn}s}
We take $\mathcal{V}$-representation formed polytopes to depict the state/output domains during the layer-wise propagation, involving sequential affine transformations and $\reluop$ activation functions. Beginning with a polytope with the vertex matrix $\bm{V}\in \mathbb{R}^{M\times D}$, where $M$ and $D$ respectively represent the number of vertices and the vertex dimension, the affine transformation essentially imposes a vertex-wise operation, namely,
$\bm{W}\cdot \bm{V}_{i,:}+\bm{b}, i\in\{0,1,\cdots,M-1\},$
where $\bm{W}$ and $\bm{b}$ are the weight matrix and bias vector of two adjacent {\dnn} layers. The affine transformation performs exact computation and outputs a polytope for the subsequent \relu operator.

In contrast, we adopt an under-approximation computation style for \relu\ operators due to the prohibitive burden of exact computation. $\reluop$ is an element-wise operator, i.e., $\reluop(\bm{x})=\reluop_{d-1}(\bm{x}) \circ \cdots \circ \reluop_{1}(\bm{x}) \circ\reluop_{0}(\bm{x}),$ where $d$ is the dimension number of $\bm{x}$ and
 $\reluop_i(\cdot)$ is the operation applying $\reluop$ on the $(i+1)$-th dimension. Let $\bm{e}_i=(\underbrace{0,\ldots,0}_{i}, 1,0,\ldots,0)^T$ and  $\mathcal{H}_i$ be the coordinate plane $\{\bm{x}\mid\bm{x}^T\bm{e}_i=0\}$. It is notable that $\reluop_i(P)$ may contain at most two polytopes: the one ``above'' $\mathcal{H}_i$ (denoted by ${P}_i^{+}$) and the one ``within'' $\mathcal{H}_i$ (denoted by ${P}_i^{-}$), respectively resulting from the top part $\hat{P}_{i}^{+}$ and bottom part $\hat{P}_{i}^{-}$ of polytope $P$ w.r.t. hyperplane $\mathcal{H}_i$, as illustrated in Figure~\ref{relu_operator}. Although it can be observed that the dimension of ${P}_i^{+}$ is generally larger than that of ${P}_i^{-}$, indicating that the volume of ${P}_i^{-}$ compared to ${P}_i^{+}$ can be neglected, however, this may be temporary. During the subsequent computation, both affine transformations and $\reluop$ operations can potentially reduce the dimensions of polytopes, and ${P}_i^{-}$ may regain the advantageous position. Consequently, it is inadvisable to blindly discard ${P}_i^{-}$, while, considering both ${P}_i^{+}$ and ${P}_i^{-}$ simultaneously leads to an exponential increase in the number of polytopes. Therefore, instead of traversing the binary computation tree of $\reluop$ operators, we propose a random branch based under-approximation algorithm for $\reluop$ operators.

\begin{figure}[htbp]
	\centering
	\includegraphics[width=.47\textwidth]{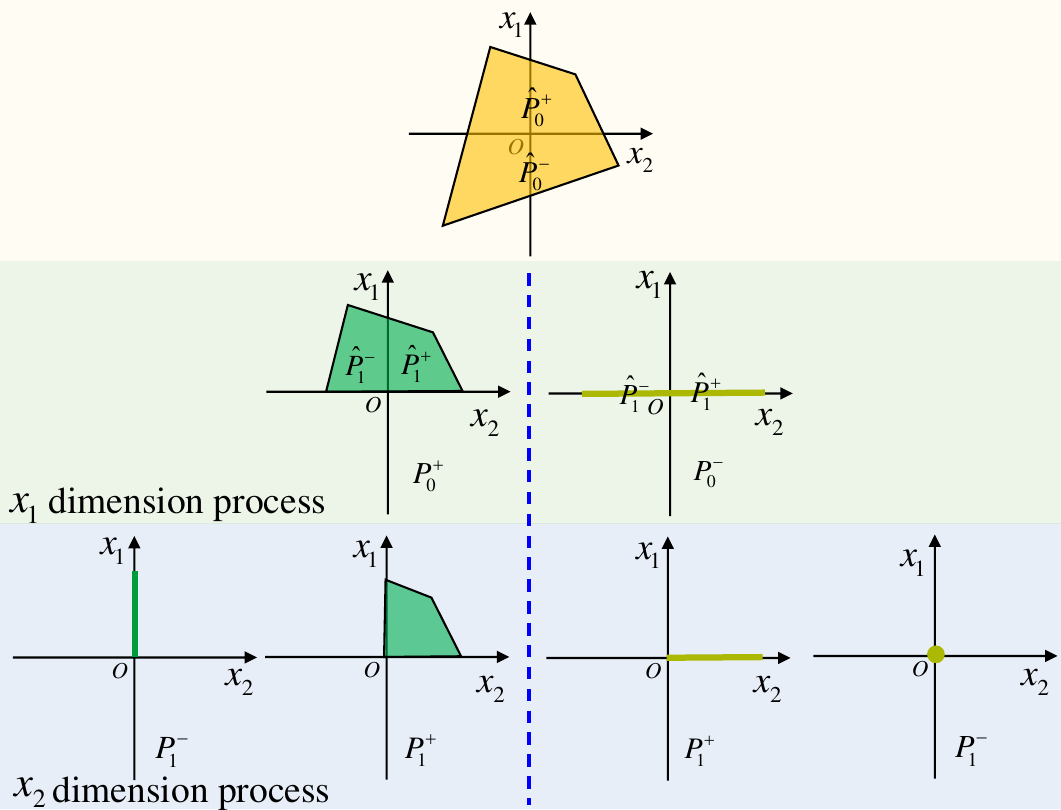}
	\caption{Demonstrations of $\reluop$ operators}
	\label{relu_operator}
\end{figure}

The main idea of under-approximation computation of $\reluop$ operators is under-approximating one of the binary tree branches randomly, which is demonstrated in Alg. \ref{alg:underapp} and carried out dimension by dimension (Line \ref{dimension-u}). When operating on the $d$-th dimension, Alg. \ref{alg:underapp} examines the coordinate values of all vertices in the dimension and categorizes them into three cases: all positive, all negative, and mixed positive and negative values. For the first two cases, Alg. \ref{alg:underapp} applies the $\reluop$ operator to the $d$-th dimension and returns the processed vertices for the $(d+1)$-th dimension (Line \ref{case}-\ref{underrelu}). In the case of mixed positive and negative values, Alg. \ref{alg:underapp} first divides the vertices into a positive vertex set and a negative vertex set based on the sign of coordinate values (Line \ref{dividepn_s}-\ref{dividepn_e}). The fundamental thought behind the under-approximation involves finding a point on the coordinate plane to replace each negative (resp. positive) vertex to under-approximate ${P}_d^{+}$ (resp. ${P}_d^{-}$). To achieve this, taking the replacement of negative vertices as an example case (Line \ref{true case}), a point set $S_{int}$ is introduced for each negative vertex, containing the intersection points between the segments connecting the negative vertex to all positive vertices and the coordinate hyperplane $\mathcal{H}_d$ (Line \ref{inter_s}-\ref{inter_e}). After computing the sets of intersection points for each negative vertex, one point from each set is selected to replace the corresponding negative vertex, with the criterion being the point with the greatest distance from the previously selected point (Line \ref{select_s}-\ref{select_e}). Finally, the union of the set of positive vertices and the vertex set with replaced vertices is returned for the subsequent dimension processing (Line \ref{union}). As for the replacement of positive vertices (Line \ref{false case}), the same process goes except for returning the set union of the projection points on $\mathcal{H}_d$ of the negative vertices and the replacement vertices (Proj$_d(\bm{x})$ is the operator that projects $\bm{x}$ to hyperplane $\mathcal{H}_d$).  Theorem \ref{thm:under} guarantees the soundness of Alg.~\ref{alg:underapp},  whose proof is provided in Appendix~\ref{proof sound}.

\begin{algorithm}[!ht]
    \renewcommand{\algorithmicrequire}{\textbf{Input:}}
    \renewcommand{\algorithmicensure}{\textbf{Output:}}
    \caption{Under-approximation of the \relu\ operator}
    \label{alg:underapp}
    \begin{algorithmic}[1] 
    \REQUIRE A polytope $P$ with vertex matrix $\bm{V}\in \mathbb{R}^{M\times D}$
    \ENSURE An under-approximation polytope with vertex matrix $\bm{V}^{u}$ of $\reluop(P)$.
        \STATE $\bm{V}^u = \bm{V}$ 
        \FOR{$d=0$ \textbf{to} $D-1$} \label{dimension-u}
        \IF {$\min(\bm{V}^{u}_{:,d})\geq 0$ \textbf{or} $\max(\bm{V}^{u}_{:,d})< 0$} \label{case}
        \STATE $\bm{V}^{u}=\text{ReLU}_d(\bm{V}^{u})$ \label{underrelu}
        \ELSE
         \STATE ${S}_p=\emptyset,{S}_n=\emptyset, S_{proj}=\emptyset, S_{can}=\emptyset, S_{r}=\emptyset$
        \FOR{$m=1$ \textbf{to} $|\bm{V}^{u}|$}  \label{dividepn_s}
        \IF{$\bm{V}^{u}_{m,d}\geq 0$}
        \STATE ${S}_p={S}_p\cup \{ \bm{V}^{u}_{m,:} \}$
        \ELSE
        \STATE  ${S}_n={S}_n\cup \{ \bm{V}^{u}_{m,:} \}$
        \STATE  ${S}_{proj}={S}_{proj}\cup \{ \text{Proj}_d{(\bm{V}^{u}_{m,:})} \}$ \label{projop}
        \ENDIF
        \ENDFOR \label{dividepn_e}
        \STATE Randomly generate a boolean variable $flag$
        \IF{$flag$} \label{true case}
        \STATE $S=S_n, T=S_p, S_{proj}=S_p$
        \ELSE
        \STATE $S=S_p, T=S_n$ \label{false case}
        \ENDIF
        \FOR{$\bm{s}$ \textbf{in} $S$} \label{inter_s}
        \STATE $S_{int}=\emptyset$
        \FOR {$\bm{t}$ \textbf{in} $T$}
        \STATE  $seg=\alpha \bm{s}+(1-\alpha) \bm{t}, \alpha\in[0,1]$
        \STATE $S_{int}=S_{int}\cup \{seg \cap \mathcal{H}_d \}$
        \ENDFOR
        \STATE $S_{can}=S_{can}\cup \{S_{int}\}$
        \ENDFOR  \label{inter_e}
        
        \FOR{$S$ \textbf{in} $S_{can}$} \label{select_s}
        \IF{$S_{r}=\emptyset$}
        \STATE randomly select $s_{r}\in S$ and $S_{r}=S_{r}\cup \{s_r\}$
        \ELSE 
        \STATE $s_{r}\in S$ that is farthest to $S_{r}$
        and $S_{r}=S_{r}\cup \{s_r\}$
        \ENDIF
        \ENDFOR \label{select_e}
        \STATE  Initialize $\bm{V}^u\in \mathbb{R}^{\#S^{u}\times D}$ with  $S^u={S}_{proj} \cup S_r$.  \label{union}
        \ENDIF
        \ENDFOR
       
        \STATE \textbf{return} $\bm{V}^u$
    \end{algorithmic}
\end{algorithm}

\begin{theorem}(Soundness guarantee)
\label{thm:under}
The polytope with vertex matrix  $\bm{V}^u$ returned by Alg. \ref{alg:underapp} is an under-approximation polytope of $\reluop(P)$.
\end{theorem}

\noindent \textbf{Complexity analysis.}
For one thing, Alg.~\ref{alg:underapp} involves processing $D$ dimensions, and the time complexity of computing the intersection points for each dimension is $O((M/2)^2)$. Therefore, the worst time complexity is $O(DM^2)$. For another thing, since Alg. \ref{alg:underapp} replaces a negative (resp. positive) vertex with an intersection point and remains the positive (resp. negative) vertices, the space complexity is $O(1)$.

\subsection{\urfornnv Verification Framework}
\label{UR4NNV}
The \urfornnv verification framework are shown in Figure~\ref{ur4nnv framework}, which is fundamentally based on multiple epochs of under-approximation reachability analysis to approximate the exact output region. As expressed in Eqn.\eqref{eq:under}, \urfornnv is more adept at falsifying {\dnn} properties. That is to say, for a property $p:=(\mathcal{X},\mathcal{Y})$ w.r.t. $\mathcal{N}$, once the reachability result $\Theta(\mathcal{X})$ returned by some epoch is not included in the desired output region $\mathcal{Y}$, the verification terminates and returns  $\mathcal{N}\not\models p$. 

Nevertheless, when the under-approximation reachability analysis reaches the specified number of executions (or timeout), and for each execution $\Theta(\mathcal{X})\subseteq \mathcal{Y}$, then \urfornnv incorporates a sample check mechanism to seek counterexamples that violate the properties from a randomly collected sample set $S$ from the input region $\mathcal{X}$ for a final speculative test. If the {\dnn} fails the sample check, i.e., the output of some sample locates outside of $\mathcal{Y}$, \urfornnv also outputs $\mathcal{N}\not\models p$. Otherwise, \urfornnv would prefer to report that the property holds on $\mathcal{N}$ with a confidence level $cl$, i.e., $\mathcal{N}\models p$ with $cl$. The confidence level is the ratio that the outputs corresponding to $S$ located in the union of the under-approximation regions resulting from Alg. \ref{alg:underapp}, i.e., $cl=\frac {\#(\mathcal{N}(\bm{s})\in {\rm cvx}(\cup_{i=1}^{N} P^{U}_i)), \bm{s}\in S}{\#S},$ where $P^{U}_i$ is the output of Alg.~\ref{alg:underapp} in the $i$-th epoch, $N$ is the epoch bound and ${\rm cvx}(\cdot)$ refers the convexhull. The completeness of the Alg.~\ref{alg:underapp} guarantees an upper confidence bound returned by \urfornnv, whose proof is provided in Appendix \ref{proof complete}.

\begin{figure}[!t]
	\centering
	\includegraphics[width=.47\textwidth]{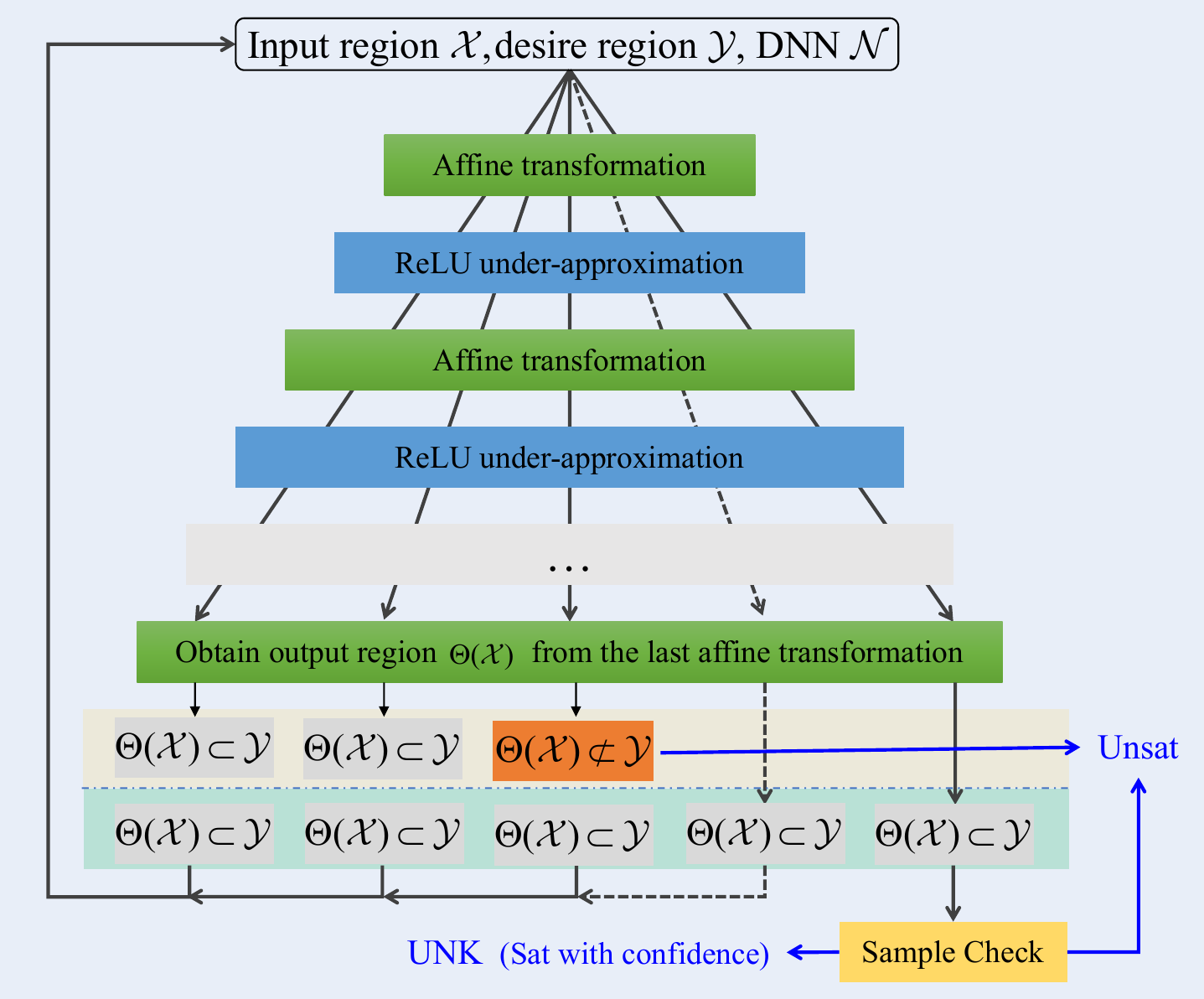}
	\caption{Workflow of \urfornnv framework}
	\label{ur4nnv framework}
\end{figure}

\begin{theorem}\label{thm: complete}(Completeness guarantee) The convex union of resultant polytopes returned by an infinite number of executions of Alg.~\ref{alg:underapp} can cover the exact output region of $\reluop$ operators, i.e., $\reluop(P) \subset {\rm cvx}(\cup_{i=1}^{\infty}P^{U}_i)$.

\end{theorem}

\subsection{Optimization Strategy}
Theorem \ref{thm:under} and \ref{thm: complete} guarantee the soundness and completeness of the proposed random branch
based under-approximation algorithm for ReLU operators. To further enhance the efficiency and effectiveness of Alg.~\ref{alg:underapp} on {\dnn} verification, several heuristic optimization strategies are presented following.
 
\textbf{Dimension-priority Assignment.} Due to randomness, the processing order of vertex dimensions can affect the execution results. On the one hand, to achieve the maximum under-approximation polytopes, the priority of the dimensions with mixed positive and negative values is positively correlated with their positive values, greedily ensuring the preservation of a larger output polytope (named maximal first strategy, MF). On the other hand, to avoid getting trapped in local solution regions, all dimensions are randomly permuted before the algorithm starts (named random first strategy, RF).

\textbf{Sub-polytope Pruning.} As shown in Figure~\ref{relu_operator}, for the $d$-th dimension, if ${P}_d^{-}\subseteq {P}_d^{+}$, it is redundant to compute ${P}_d^{-}$ and its subsequent traces. Moreover, the cases where ${P}_d^{-}\subseteq {P}_d^{+}$ can be identified that the projection points of negative vertices of polytope $P$ on hyperplane $\mathcal{H}_d$ are located within $P$. Consequently, with satisfying the projection relationship, we only need to calculate ${P}_d^{+}$, termed complete top polytope strategy (CTP). Another more straightforward pruning strategy is to greedily and shortsightedly compute  ${P}_d^{+}$ and discard ${P}_d^{-}$ directly for dimension $d$, sacrificing completeness while avoiding condition judgment, called top polytope strategy (TP).

\textbf{Maximal Under-approximation.} During selecting replacement points for positive/negative vertices, it is recommended to make multiple selections and maintain the set of replacement points with the greatest distance. The maximal under-approximation strategy (MUA) is consistently available and helps approach the exact output region more closely.

\textbf{Parallel Execution.} The multi-epoch based verification paradigm of \urfornnv naturally supports parallel execution, and we recommend adopting parallelism, as it allows for further reduction in verification time.

\section{Experiments}
\label{exp}
In this section, we take the ACAS Xu {\dnn}s as the study case, which is a widely recognized and extensively evaluated {\dnn} safety verification problem~\cite{2019acasxu,10.1007/978-3-030-53288-8_1}. ACAS Xu {\dnn}s contain an array of 45 {\dnn}s (organized as a $5 \times 9$ {\dnn} array, $\mathcal{N}_{1,1},\cdots,\mathcal{N}_{1,9},\mathcal{N}_{2,1},\cdots,\mathcal{N}_{5,9}$) that produces maneuver advisories of the unmanned version of Airborne Collision Avoidance System X, which is a highly safety-critical system developed by the Federal Aviation Administration. In detail, these {\dnn}s take a five-dimensional input, representing the scenario around the aircraft, and output a five-dimensional prediction, indicating five possible advisories. These ACAS Xu {\dnn}s are associated with ten safety properties, $\phi_1,\ \phi_2,\ \cdots,\ \phi_{10}$, requiring that the output corresponding to an input within a specific input region must fall in the given safe region(s). The safety description, formal definition, and test {\dnn}s of these properties are supplemented in Appendix \ref{DNN property}, together with the DNN structure.

All the experiments herein are carried out on a 16-core (AWS r5d.4xlarge) machine with 120 GB of memory, equipped with Ubuntu 20.04 and Matlab 2022b. 

\subsection{Falsifying {\dnn} Property}

As clarified in Section \ref{UR4NNV}, the \urfornnv verification framework is strongly recommended to falsify a {\dnn} property. Consequently, in this subsection, we demonstrate the effectiveness and efficiency of \urfornnv in terms of falsifying {\dnn} properties, which means that \urfornnv returns the deterministic \textsc{Unsafe} conclusion. 

\begin{table}[!t]
\centering
\caption{
{ Comparison with state-of-the-art verification methods.}}
\label{exp:false}
\begin{footnotesize}
\setlength{\tabcolsep}{0.01mm}{
\begin{tabular}{c|clcccc}
    \toprule
    \multicolumn{1}{c}{\multirow{2}{*}{\textbf{Paradigm}}} & \multirow{2}{*}{\textbf{Method}} & \multicolumn{1}{c}{\multirow{2}{*}{\textbf{Task}}} & \multicolumn{3}{c}{\textbf{VC}}                                                              & \multicolumn{1}{c}{\multirow{2}{*}{\textbf{VT}}} \\ \cmidrule(lr){4-6}
\multicolumn{1}{c}{}                                   &                                  & \multicolumn{1}{c}{}                               & \multicolumn{1}{c}{\textbf{\safe\ }} & \multicolumn{1}{c}{\textbf{\unsafe\ }} & \multicolumn{1}{c}{\textbf{\UNK}} & \multicolumn{1}{c}{}                             \\ 
    \hline
    \multirow{10}{*}{\makecell[c]{SMT Solver \\ Based \\  Verification}} 
      &  \multirow{5}{*}{\makecell[c]{Reluplex\\ \textcolor{blue}{41/44}}}  & $\phi_2$ 36 NNs &  1 &  34 & 1  & 194,478 \\
     &   & $\phi_3$ 3 NNs  &  0 & 3  &  0 &  11.58\\
     &   & $\phi_4$ 3 NNs  &  0 & 3  &  0 & 12.35 \\
     &  & $\phi_7$ 1 NN  &  0  & 0  & 1  & Timeout\\
    &  &  $\phi_8$ 1 NN &  0  & 0  & 1  & Timeout \\
    
    \cline{2-7}
 &  \multirow{5}{*}{\makecell[c]{Marabou/ \\
Marabou \\ DnC  \\ \textcolor{red}{44/44}}}  & $\phi_2$ 36 NNs &  1 &  35 & 0  & \textcolor{blue}{43,892} \\
     &   & $\phi_3$ 3 NNs  &  0 & 3  &  0 &  \textcolor{blue}{7.12}\\
     &   & $\phi_4$ 3 NNs  &  0 & 3  &  0 & \textcolor{blue}{8.81} \\
     &  & $\phi_7$ 1 NN  &  0  & 1  & 0  & \textcolor{red}{4,847}\\
    &  &  $\phi_8$ 1 NN &  0  & 1  & 0  & \textcolor{blue}{3,761} \\
    \hline
     \multirow{5}{*}{\makecell[c]{Exact \\ Reachability \\ Based \\ Verification}} 
      &  \multirow{5}{*}{\makecell[c]{NNV \\ Exact \\ Star \\ 38/44}}    & $\phi_2$  36 NNs &  1 &  29 & 6  & 403,472 \\
     &   & $\phi_3$ 3 NNs  &  0 & 3  &  0 & 38,537 \\
     &   & $\phi_4$ 3 NNs  &  0 & 3  &  0 & 40,061 \\
     &  & $\phi_7$ 1 NN  &  0  & 1  & 0  & 13,244\\
    &  &  $\phi_8$ 1 NN &  0  & 1  & 0  & 11,350 \\
    \hline

    \multirow{15}{*}{\makecell[c]{Over-appr. \\ Reachability \\ Based \\ Verification}} 
      &  \multirow{5}{*}{\makecell[c]{Zonotope \\ \\ 0/44}}   & $\phi_2$  36 NNs &  0 &  0 & 36  &  0.87 \\
     &   & $\phi_3$ 3 NNs  &  0 & 0  &  3 &  0.08\\
     &   & $\phi_4$ 3 NNs  &  0 & 0  &  3 & 0.08 \\
     &  & $\phi_7$ 1 NN  &  0  & 0  & 1  & 0.13 \\
    &  &  $\phi_8$ 1 NN &  0  & 0  & 1  & 0.23 \\
    
    \cline{2-7}
 &  \multirow{5}{*}{\makecell[c]{Abstract\\ Domain \\ \\ 0/44}}  & $\phi_2$  36 NNs &  0 &  0 & 36  & 3.60 \\
     &   & $\phi_3$ 3 NNs  &  0 & 0  &  3 & 0.25 \\
     &   & $\phi_4$ 3 NNs  &  0 & 0  &  3 & 0.17 \\
     &  & $\phi_7$ 1 NN  &  0  & 0  & 1  & 0.27 \\
    &  &  $\phi_8$ 1 NN &  0  & 0  & 1  & 0.22 \\

    \cline{2-7}
 &  \multirow{5}{*}{\makecell[c]{NNV \\ Appr. \\ Star \\ 0/44}}  & $\phi_2$  36 NNs &  0 &  0 & 36  &  746.23 \\
     &   & $\phi_3$ 3 NNs  &  0 & 0  &  3 & 0.55 \\
     &   & $\phi_4$ 3 NNs  &  0 & 0  &  3 & 0.63 \\
     &  & $\phi_7$ 1 NN  &  0  & 0  & 1  & 32.46 \\
    &  &  $\phi_8$ 1 NN &  0  & 0  & 1  & 17.76 \\
    
    \hline
      
      \multirow{5}{*}{\makecell[c]{ Under-appr. \\ Reachability \\ Based \\ Verification}} 
      &  \multirow{5}{*}{\makecell[c]{\textbf{\urfornnv}\\ \textbf{39/44}} } & $\phi_2$  36 NNs &  0 &  32 & 4  &  \textcolor{red}{1,027.04$\pm$40.63}\\
     &   & $\phi_3$ 3 NNs  &  0 & 3  &  0 & \textcolor{red}{0.13$\pm$0.01} \\
     &   & $\phi_4$ 3 NNs  &  0 & 3  &  0 & \textcolor{red}{0.14$\pm$0.01}  \\
     &  & $\phi_7$ 1 NN  &  0  & 0  & 1  & Timeout \\ 
    &  &  $\phi_8$ 1 NN &  0  & 1  & 0  & \textcolor{red}{5.82$\pm$4.87} \\
    
    \bottomrule
\end{tabular}}
\begin{tablenotes}
\item Note: timeout deadlines for Reluplex, NNV Exact Star and \urfornnv are set as 12 hours, 12 hours and 1 minute respectively. 
\end{tablenotes}
\end{footnotesize}

\end{table}

We compare \urfornnv with well-known {\dnn} verification methods, including Reluplex~\cite{10.1007/978-3-030-25540-4_26}, Marabou/Marabou DnC~\cite{10.1007/978-3-319-63387-9_5}, NNV Exact Star~\cite{10.1145/3358230,10.1007/978-3-030-30942-8_39,10.1007/978-3-030-53288-8_1}, Zonotope~\cite{10.5555/3327546.3327739}, Abstract domain~\cite{10.1145/3290354} and NNV approximation star~\cite{10.1007/978-3-030-53288-8_1}. Among these comparison baselines, Reluplex and Marabou/Marabou DnC are SMT solver based verification methods, and NNV exact star verifies {\dnn} property with exact reachability analysis. All these three methods can provide theoretically sound and complete verification conclusions, which can be regarded as ground truth. Moreover, Zonotope, Abstract domain and NNV approximation star are over-approximation reachability based verification strategies. These methods are all implemented in the NNV verification tool~\cite{10.1007/978-3-030-53288-8_1}, and we refer readers to Section \ref{work} for a detailed introduction of these related verification frameworks.

\begin{table}[ht]
\caption{Performance of integrating \urfornnv with existing methods}
\label{improve_result}
\centering
\begin{footnotesize}
\setlength{\tabcolsep}{2.5mm}{
\begin{tabular}{ccccc}
\toprule
\multirow{2}{*}{\makecell[c]{\textbf{Method(X)}}} &\multicolumn{2}{c}{\textbf{X}}& \multicolumn{2}{c}{\textbf{\urfornnv+X}} \\ \cmidrule(lr){2-3} \cmidrule(lr){4-5}  
                 &  Tasks &  VT   & Tasks &  VT     \\
\hline
Reluplex & 41/44 &     280,901    &   42/44  &   65,946 \\ 
Marabou &  44/44 &    52,515   &    44/44   &  10,861  \\
NNV Exact Star & 38/44  &  506,664     &  40/44  & 59,212  \\
\hdashline
Zonotope & 0/44  &    1.39   &     39/44  &  1,140  \\
Abstract Domain &  0/44 &   4.51    &    39/44   &  1,143   \\
NNV Appr. Star & 0/44  &   798    &    39/44   &  1,936   \\

 \bottomrule                   
\end{tabular}}
\end{footnotesize}
\end{table}

Due to its inherent stochastic nature, \urfornnv is independently tested 20 times for various combinations of properties and {\dnn}s. Once one execution reports \textsc{Unsafe}, the verification conclusion (VC) is marked as \unsafe; otherwise, it is \UNK (timeout). The comparison results are illustrated in Table \ref{exp:false}. For one thing, verification based on SMT solver (Reluplex 41/44, Marabou 44/44) or exact reachability (38/44) can verify properties on almost all tasks, except for a few timeouts. Over-approximation based verification strategies fail on all tasks because of the unknown dilemma. Thus, they return the \UNK result. \urfornnv completes 39 verification tasks out of 44. For another thing, the time consumption of the verification frameworks based on SMT solver or exact reachability is significant, with a magnitude of 1-2 orders, more than that of \urfornnv. Although over-approximation verification costs much shorter time, all the verification tasks failed. The optimal and second-best results are respectively in red and blue.

When using \urfornnv as a preliminary preprocessing for existing {\dnn} verification frameworks, namely, applying \urfornnv to verify {\dnn} properties first, any properties that \urfornnv is unable to verify are then passed on to existing frameworks for further verification. Table \ref{improve_result} presents the completion and verification time of the original frameworks and the \urfornnv-integrated ones on the tasks listed in Table \ref{exp:false}. For exact verification methods, the verification strategy incorporating \urfornnv significantly reduces the verification time, ranging from about 76.5\%, 79.3\% to 88.3\%, while also improving the task completion slightly. As for the over-approximation based verification frameworks, the task completion is entirely attributed to \urfornnv, and the associated verification overhead is considerable. This also indicates that \urfornnv alleviates the ``unknown dilemma" to a great extent.

\subsection{Validating {\dnn} Property with Confidence}
When \urfornnv fails to falsify {\dnn} properties, it turns to provide a confidence level on the satisfiability of properties, and we evaluate this confidence-based DNN verification in this subsection. Recall that the \urfornnv verification framework can provide theoretical completeness under sufficient verification epochs. However, it is difficult to achieve such completeness mainly because Alg.~\ref{alg:underapp} is essentially a binary tree random search procedure, and the search space exponentially increases with the layer dimension and the {\dnn} depth.

\begin{table*}[ht]
\caption{Comparisons among different optimization strategies}
\label{Compare_result}
\centering
\begin{footnotesize}
\setlength{\tabcolsep}{1.2mm}{
\begin{tabular}{ccccccccccccccccc}
\toprule
\multirow{2}{*}{\makecell[c]{\textbf{Property}}} &\multirow{2}{*}{\makecell[c]{\textbf{{\dnn}}}} & \multicolumn{3}{c}{\textbf{Pure}} & \multicolumn{3}{c}{\textbf{RF+TP}} & \multicolumn{3}{c}{\textbf{RF+CTP}} & \multicolumn{3}{c}{\textbf{MF+TP}} & \multicolumn{3}{c}{\textbf{MF+CTP}} \\ \cmidrule(lr){3-5} \cmidrule(lr){6-8} \cmidrule(lr){9-11} \cmidrule(lr){12-14} \cmidrule(lr){15-17}  
                 &   &  VC   & VE         & VT      & VC  & VE         & VT     & VC   & VE         & VT     & VC   & VE         & VT    & VC    & VE         & VT       \\
\hline
\multirow{7}{*}{\makecell[c]{$\phi_2$}} & $\mathcal{N}_{2,1}$ &  \UNK        &   --     &   --      &   \unsafe(14)     &   323      &   42     &    \unsafe(2)      &   120      &   57      &   \textcolor{red}{\unsafe(20)} &      \textcolor{red}{110}   &  \textcolor{red}{7}       & \textcolor{blue}{\unsafe(17)} &     \textcolor{blue}{14}    &       \textcolor{blue}{29}     \\

&  $\mathcal{N}_{2,3}$ &    \UNK      &    --   &    --      &   \textcolor{red}{\unsafe(19)}      &   \textcolor{red}{223}      &  \textcolor{red}{31}        &  \textcolor{blue}{\unsafe(16)}        &     \textcolor{blue}{59}    &     \textcolor{blue}{49}    &    \UNK &  -- &  --  & \UNK  &  -- &  --  \\
&  $\mathcal{N}_{3,5}$ &     \UNK     &     --    &     --     &    \textcolor{red}{\unsafe(20)}    &   \textcolor{red}{ 44}     &  \textcolor{red}{6 }      &      \unsafe(12)     &   32    &  41       &  \unsafe(7)   &  27   & 64 &  \UNK     &     --    &     --  \\
 
 &  $\mathcal{N}_{3,7}$ &     \UNK     &   --     &  --        &    \textcolor{red}{\unsafe(20)}     &     \textcolor{red}{115}    &     \textcolor{red}{14}    &          \textcolor{blue}{ \unsafe(14) } &  \textcolor{blue}{ 91}      &    \textcolor{blue}{36}     &  \UNK  &   -- &  --  & \UNK  &   -- &  --   \\
 &  $\mathcal{N}_{4,3}$ &    \unsafe(14)     &   409   &     38    &      \textcolor{red}{\unsafe(20)}    &   \textcolor{red}{24}      &   \textcolor{red}{3}      &    \unsafe(18)      &       32  &     33    &  \unsafe(1) &  1109 & 60  &  \textcolor{blue}{\unsafe(20)} &   \textcolor{blue}{26}    &     \textcolor{blue}{6}   \\
 &  $\mathcal{N}_{5,4}$ &     \textcolor{red}{\unsafe(20)}     &     \textcolor{red}{32}   &      \textcolor{red}{3}    &   \textcolor{blue} {\unsafe(8)}     &     \textcolor{blue}{326}    &      \textcolor{blue}{48}   &   \unsafe(6) &     125    &    54     & \UNK  & --  &  -- & \UNK  & --  &  --     \\

 &  $\mathcal{N}_{5,9}$ &     \UNK     &    --    &      --     &    \textcolor{red}{\unsafe(20)}     &    \textcolor{red}{39}     &     \textcolor{red}{5}    &   \unsafe(5) &     52    &     71    & \UNK  & --  &  -- & \UNK  & --  &  --     \\

 $\phi_3$ &  $\mathcal{N}_{1,7}$ &   \unsafe(20)       &   1     &    0.09      &    \textcolor{blue}{\unsafe(20)}    &     \textcolor{blue}{1}    &    \textcolor{blue}{0.09}     &   \unsafe(20)       &    1   &   0.21         &  \textcolor{red}{\unsafe(20)}  &  \textcolor{red}{1} & \textcolor{red}{0.04} & \unsafe(20) & 1 & 0.20 \\

  $\phi_4$ &  $\mathcal{N}_{1,9}$ &   \unsafe(20)       &   1     &    0.08      &    \textcolor{blue}{\unsafe(20) }   &     \textcolor{blue}{1}   &    \textcolor{blue}{0.08}    &  { \unsafe(20)}       &    1   &   0.24&  \textcolor{red}{\unsafe(20)}  &  \textcolor{red}{1} & \textcolor{red}{0.02} & \unsafe(20) & 1 & 0.48 \\

 $\phi_8$ &  $\mathcal{N}_{2,9}$ &   \unsafe(11)       &   306     &    38      &    \textcolor{red}{\unsafe(20)}    &     \textcolor{red}{45}    &    \textcolor{red}{6}     &   \textcolor{blue}{\unsafe(15) }      &    \textcolor{blue}{26}   &   \textcolor{blue}{36}         &  \unsafe(11)  &  578 & 45 & \unsafe(2) & 12 & 85 \\
 \bottomrule                   
\end{tabular}}
\begin{tablenotes}
\item  Note: \UNK and  \unsafe are respectively short for \textsc{Unknown} and \textsc{Unsafe}. 
\end{tablenotes}
\end{footnotesize}
\end{table*}

Figure~\ref{verify_confidence} displays the mean and standard deviation of the confidence levels on the satisfaction of property $\phi_1, \phi_6, \phi_7$ w.r.t. $\mathcal{N}_{1,2}, \mathcal{N}_{1,1}, \mathcal{N}_{1,9}$ respectively, among 20 independent evaluations. It can be observed that the $\phi_1$ holds on $\mathcal{N}_{1,2}$ with relatively high confidence, approaching 80\%, while the confidence levels of the other two cases are relatively low, around 40\%. As the verification epoch increases, the confidence level's mean and standard deviation tend to stabilize. The stabilization of the standard deviation stems from the reduction of randomness, and that of the mean can be explained by the exponential small sub-polytopes within the exact output region and simply increasing the epochs is still insufficient to cover the output region effectively. Therefore, due to the stochastic nature of the verification framework and the exponential number of sub-spaces, the confidence level returned by \urfornnv serves as a preliminary reference, and \urfornnv is highly recommended for property falsification.

\begin{figure}[!htb]
\centering
\includegraphics[width=.35\textwidth]{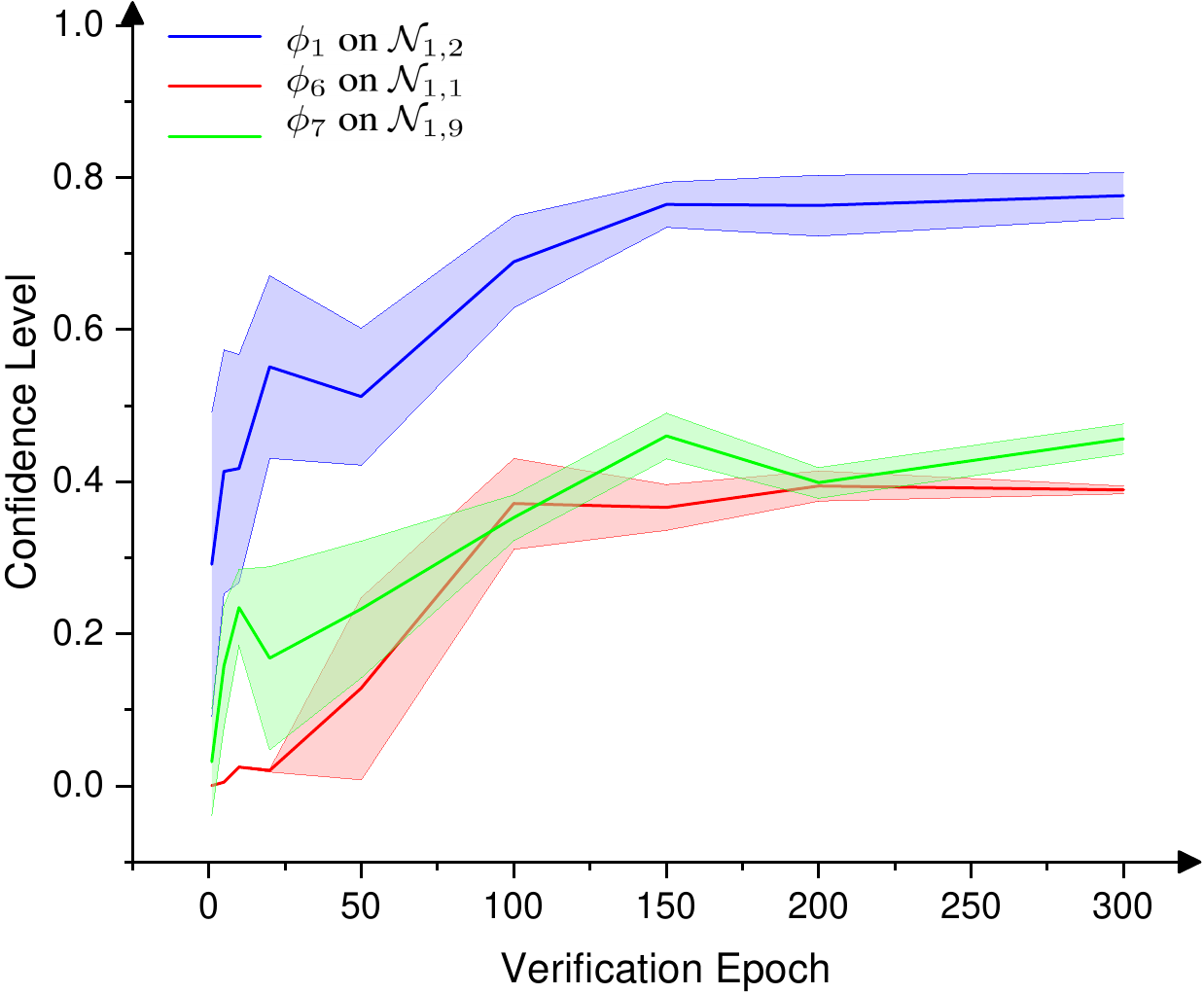}
\caption{Verifying properties with confidence levels}
\label{verify_confidence}	

\end{figure}

\subsection{Optimization Strategy Comparison}

This subsection compares the performance of the proposed optimization strategies for \urfornnv. Notably, all the optimization strategies are essentially designed for Alg. \ref{alg:underapp} and the positive effect of the maximal under-approximation and parallel execution is evident. Consequently, we mainly focus on two categories of optimization strategies: dimension-priority assignment and sub-polytope pruning and we conduct experiments by combining the MF/RF and CTP/TP optimization strategies on Alg. \ref{alg:underapp} (named Pure version), comparing the effectiveness and efficiency based on verification conclusions (VC), verification epochs (VE) and verification time (VT). 

 We select 4 properties across 10 {\dnn}s from Table \ref{exp:false} that were successfully falsified as test cases. Each strategy is independently tested 20 times for different cases and the experimental results are shown in Table \ref{Compare_result}. If all tests fail to falsify the property, the verification conclusion (VC) is marked as \UNK. Otherwise, the number of successful validations is recorded in parentheses after ``\unsafe". The average number of verification epochs and verification time over 20 tests are also recorded in Table \ref{Compare_result}. Among performance of different strategies,  the optimal result is marked in red and the second best in blue.  It can be seen that the optimization strategies significantly improved the pure algorithm’s inability to falsify {\dnn} properties, and the verification time is generally reduced. Overall, we recommend the combination optimization strategy of RF+TP, with the second-best being RF+CTP. However, it is undeniable that the other two strategies may have more evident optimization effects in some cases.
\section{Related Work}
\label{work}
There have been a myriad of work ~\cite{huang2019reachnn,ivanov2020verifying,xiang2018output,10.1007/s11390-020-0546-7,JCST-2207-12703,DBLP:journals/jcst/LiuSZW20} on the formal verification of DNNs, roughly categorized into \textit{sound and complete verification methods} and \textit{sound and in complete ones}. The sound and complete methods are mainly established on SAT/SMT solvers or exact reachability analysis. For \sat/\smt solver-based strategies, Katz
et al.~\cite{10.1007/978-3-319-63387-9_5} and Ehlers~\cite{ehlers2017formal} independently put forward Reluplex and Planet verifiers, two \smt solvers to verify ReLU {\dnn}s on properties expressed with SMT constraints, which both employed the \dpll algorithm~\cite{10.1145/368273.368557} to split different cases and exclude conflicting clauses. Building upon Reluplex, Katz et al.~\cite{10.1007/978-3-030-25540-4_26} later implemented the Marabou (or Marabou DnC) framework, which is no longer limited to $\reluop$ activations and can verify both fully-connected and convolutional {\dnn}s. For the exact reachability analysis based ones, Xiang et al.~\cite{DBLP:journals/corr/abs-1712-08163} and Tran et al.~\cite{10.1145/3358230,10.1007/978-3-030-30942-8_39} respectively computed the exact output regions with polytopes and star sets. Although these methods can provide complete and deterministic verification conclusions, the verification duration is computationally prohibitive and their scalability is greatly limited.  

As for the sound and incomplete methods, over-approximated abstract interpretation attracted much more attention, which propagates abstract domains layer by layer in an over-approximate way \cite{cousot1977abstract}, such as intervals \cite{wang2018efficient}, zonotopes and star sets \cite{10.1007/978-3-030-30942-8_39}. AI$^2$~\cite{gehr2018ai2} is a representative method based on zonotope abstract domains, which can verify {\dnn}s with piecewise linear activation functions. Subsequently, it was further developed to obtain tighter results via improving abstract transformation on {\dnn}s with other activation functions, like $\tanh$ and $\sigmoid$~\cite{10.5555/3327546.3327739}. Later, Singh et al.~\cite{10.1145/3290354} proposed an abstract domain that combines floating point polyhedra with intervals to over-approximate exact output regions more closely. Subsequently, Yang et al.~\cite{yang2021improving} presented a spurious region-guided strategy to refine the output regions generated by \cite{10.1145/3290354}. Moreover, Liang et al.~\cite{10.1007/978-3-031-35257-7_15} focused on the topological properties established within {\dnn}s to reduce the warping effect of over-approximation propagation based on set-boundary analysis. 


\section{Conclusion}
\label{concl}

To alleviate the unknown dilemma encountered in over-approximation {\dnn} verification, this paper proposes the \urfornnv verification framework based on under-approximation reachability analysis, which is the first utilization of under-approximation in {\dnn} verification, as far as we know. The essence of \urfornnv lies in the under-approximation of $\reluop$ activations in means of randomly searching binary tree branches. Unlike existing verification strategies, \urfornnv is more proficient in falsifying {\dnn} properties and comparison experiments demonstrate its effectiveness and efficiency. In the future, we will consider optimizing the computation efficiency of the under-approximation algorithm and extending the \urfornnv framework to {\dnn}s with other activation functions.

\section*{Ethical Statement}

There are no ethical issues.


\bibliographystyle{named}
\bibliography{reference}

\newpage

\renewcommand\thesubsection{\Alph{subsection}}
\setcounter{algorithm}{0}

\section*{Appendix}

\subsection{Proof of Theorem \ref{thm:under}}
\label{proof sound}

\textbf{Theorem 1} \textit{The polytope with vertex matrix  $\bm{V}^u$ returned by Alg. \ref{alg:underapp} is an under-approximation polytope of $\reluop(P)$.}
\begin{proof}

For a specific dimension $d$, it is notable that $\reluop_d(P)$ may contain at most two polytopes: the one ``above'' $\mathcal{H}_d$ (denoted by ${P}_d^{+}$) and the one ``within'' $\mathcal{H}_d$ (denoted by ${P}_i^{-}$), respectively resulting from the top part and bottom part of polytope $P$ w.r.t. hyperplane $\mathcal{H}_d$.

According to the definition of $\reluop$ operator, the vertex set of sub-polytope ${P}_d^{+}$ is composed of the vertices above hyperplane $\mathcal{H}_d$ and the intersection points between $P$ and $\mathcal{H}_d$, and the vertex set of sub-polytope ${P}_d^{-}$ consists of the projection of the vertices below hyperplane $\mathcal{H}_d$ and the intersection points between $P$ and $\mathcal{H}_d$ (maybe the points are redundant in some cases but they must contain the ground-truth vertex set). 

Next we show that the sub-polytope returned by Alg.~\ref{alg:underapp} is completely contained in either ${P}_d^{+}$ or ${P}_d^{-}$. It requires to show that the vertex set $S^u=S_{proj}\cup S_r$ is located in  either ${P}_d^{+}$ or ${P}_d^{-}$.

For the first case $S=S_n, T=S_p, S_{proj}=S_p$, it means that Alg.~\ref{alg:underapp} replaces the negative vertex set with $S_r$ and remains the positive vertex set $S_p$. 
\begin{itemize}
    \item Obvious, for $\bm{s}\in S_{proj} (i.e., S_p)$, $\bm{s}$ is a vertex above $\mathcal{H}_d$ and it is a vertex of resultant sub-polytope ${P}_d^{+}$;
    \item  For a point $\bm{s}\in S_r$, it is computed with a segment $Seg$ and the hyperplane $\mathcal{H}_d$. Since $Seg$ is represented by a positive vertex and a negative vertex of polytope $P$, the intersection between $Seg$ and $\mathcal{H}_d$  locates on $P$ (Convexity) and  $\mathcal{H}_d$ (Intersection) simultaneously. Thus, $\bm{s}$ lies in the sub-polytope ${P}_d^{+}$.
\end{itemize}

For the other case  $S=S_n, T=S_p$, it means that Alg.~\ref{alg:underapp} replaces the positive vertex set with $S_r$ and remains the projection of the negative vertex set $S_p$. 
\begin{itemize}
    \item For $\bm{s}\in S_{proj}$, $\bm{s}$ is the projection of a negative vertex $\mathcal{H}_d$ and it is a vertex of resultant sub-polytope ${P}_d^{-}$;
    \item  For a point $\bm{s}\in S_r$, it is computed with a segment $Seg$ and the hyperplane $\mathcal{H}_d$. Since $Seg$ is represented by a positive vertex and a negative vertex of polytope $P$, the intersection between $Seg$ and $\mathcal{H}_d$  locates on $P$ (Convexity) and  $\mathcal{H}_d$ (Intersection) simultaneously. Thus, $\bm{s}$ lies in the sub-polytope ${P}_d^{-}$.
\end{itemize}

Then, $S_{proj}\cup S_r$ is located in  either ${P}_d^{+}$ or ${P}_d^{-}$, and the containment that  $\alpha\bm{s}+(1-\alpha)\bm{t} \subset {P}_d^{+}$ or $\alpha \bm{s}+(1-\alpha)\bm{t} \subset {P}_d^{-}$, $0\le \alpha \le 1, \bm{s},\bm{t}\in S_{proj}\cup S_r$, is immediate from the convexity. That is to say, the polytope with vertex matrix  $\bm{V}^u$ (initialized with $S^u$) is an under-approximation polytope of $\reluop_d(P)$ and it is an immediate conclusion for the $\reluop$ operator over all dimensions.
\end{proof}

\subsection{Proof of Theorem \ref{thm: complete}}
\label{proof complete}

\textbf{Theorem 2} \textit{The convex union of resultant polytopes returned by an infinite number of executions of Alg.~\ref{alg:underapp} can cover the exact output region of $\reluop$ operators, i.e., $\reluop(P) \subset {\rm cvx}(\cup_{i=1}^{\infty}P^{U}_i)$.}
\begin{proof}

Proving Theorem \ref{thm: complete} is equivalent to show that every point in $\reluop(P)$ can be covered. Further, it is equivalent to show that every vertex of the sub-polytopes (a set of polytopes) resulted from $P$ under $\reluop$ operators can be covered. 

Without loss of generality, we take $\reluop_d(P)$ as an example. It is required to show that every vertex of ${P}_d^{+}$ and ${P}_d^{-}$ can be covered in some epoch and the convex union can cover their convex combination.

For the case of ${P}_d^{+}$, the vertex set of sub-polytope ${P}_d^{+}$ is composed of the vertices above hyperplane $\mathcal{H}_d$ and the intersection points between $P$ and $\mathcal{H}_d$. The former set is the set $S_{proj}=S_p$.  Every point of latter set must be the intersection point between the segment connecting a positive vertex and a negative vertex  
 and the hyperplane $\mathcal{H}_d$ or the convex combination of two intersection points. The intersection point can be chosen from a set $S\in S_{can}$ and their convex combination is implemented by the final $\rm cvx(\cdot)$ operator, thus the latter set and their combination can be covered.

For the case of ${P}_d^{-}$, the vertex set of sub-polytope ${P}_d^{-}$ consists of the projection of the vertices below hyperplane $\mathcal{H}_d$ and the intersection points between $P$ and $\mathcal{H}_d$. The former set is the set $S_{proj}$. As for the latter set, it is the same as that of the above case.

Since all the possible vertices can be covered in some epoch, their convexhull can eventually cover the sub-polytope ${P}_d^{+}$ and ${P}_d^{-}$. It is easy to get a conclusion for the $\reluop$ operator over all dimensions.
\end{proof}

\subsection{ACAS Xu DNNs and Safety Properties}
\label{DNN property}
\begin{figure*}[htbp]
    \centering
     \includegraphics[scale=0.53]{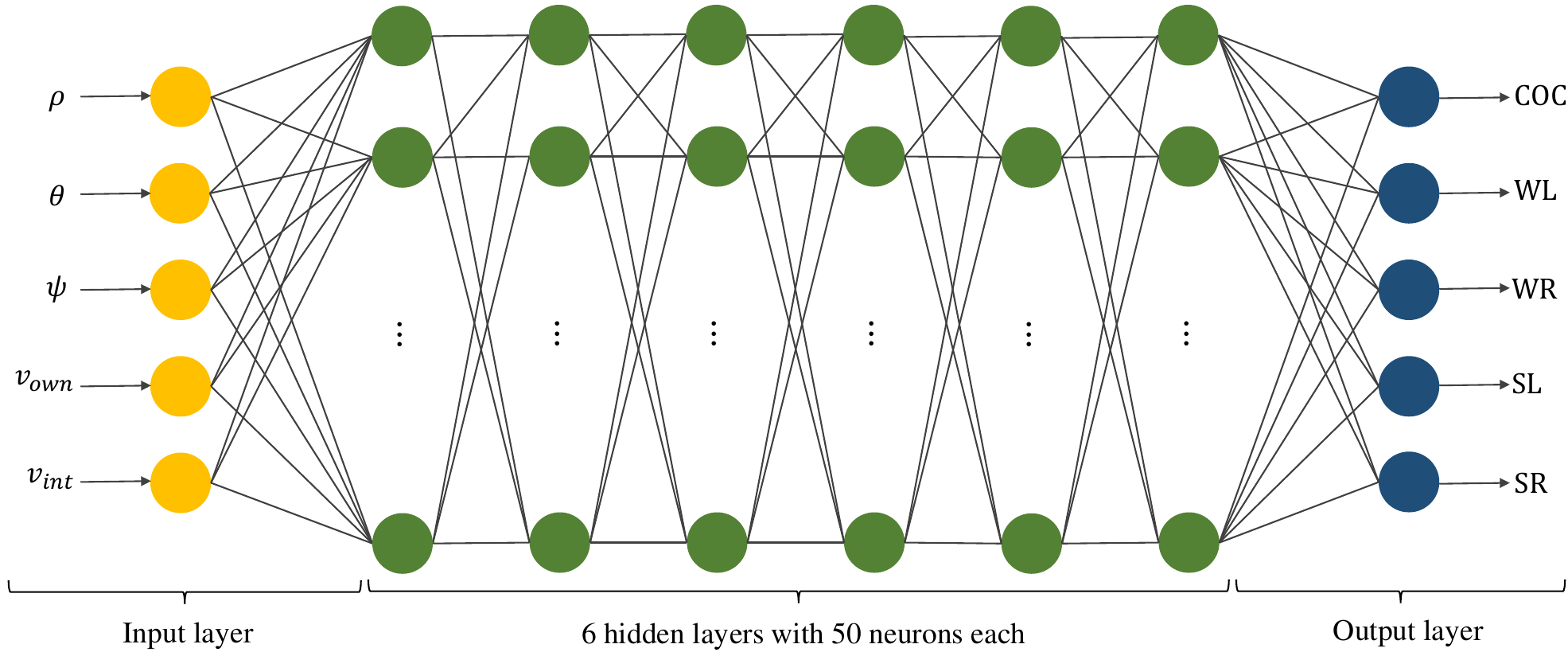}
    \caption{Depiction of the ACAS Xu DNN}
    \label{ACAS Xu DNN}
\end{figure*}
The structure of ACAS Xu DNNs is demonstrated in Fig.~\ref{ACAS Xu DNN}, input units for the ACAS Xu DNNs' inputs are $\rho$ (in feet), $\theta,\psi$ (in radians), $v_\text{own}$ and $v_\text{int}$ (in feet per second). $\theta$ and $\psi$ are measured counterclockwise, and are always in the range $[-\pi,\pi]$. The output units are the scores for the action Clear-of-Conflict (COC), weak left (WL), weak right (WR), strong left (SL) and strong right (SR). In line with the introduction in Section~\ref{exp}, the family of 45 ACAS Xu
DNNs are indexed according to the previous action $a_\text{prev}$ and
time until loss of vertical separation $\tau$ (in seconds), where $a_\text{prev}\in$ $[$Clear-of-Conflict, weak left, weak right, strong left, strong right$]$ and $\tau\in$ $[0,1,5,10,20,40,60,80,100]$.

We use $\mathcal{N}_{x,y}$ to denote the network trained for the
$x$-th value of $a_\text{prev}$ and $y$-th value of $\tau$. For
example, $\mathcal{N}_{2,3}$ is the network trained for the case where
$a_\text{prev}=$  weak left
and $\tau=5$. Using this notation, we describe and formally define the ten safety properties $\phi_1,\phi_2,\cdots,\phi_{10}$ that we tested, which are respectively shown in Table~\ref{tab:DNNproperty} and Table~\ref{tab:formalPro}.

\begin{table*}[htbp]
    \centering
        \caption{Description of the safety properties defined on  ACAS Xu DNNs.}
    \label{tab:DNNproperty}
    \begin{tabularx}{\textwidth}{lX}
    \toprule
       \textbf{Property}  & \textbf{Description} \\
       \hline
       $\phi_1$  & If the intruder is distant and is
  significantly slower than the ownship, the score of a COC advisory will
  always be below a certain fixed threshold. \\
   $\phi_2$  &  If the intruder is distant and is
  significantly slower than the ownship, the score of a COC advisory will
  never be maximal. \\
   $\phi_3$  &  If the intruder is directly ahead and is moving towards the ownship,
  the score for COC will not be minimal.\\
   $\phi_4$  &  If the intruder is directly ahead and is moving away from the
  ownship but at a lower speed than that of the ownship,
  the score for COC will not be minimal.\\
   $\phi_5$  &  If the intruder is near and approaching from the
  left, the network advises ``strong right''.\\
   $\phi_6$  &  If the intruder is sufficiently far away,
  the network advises COC.\\
   $\phi_7$  &If vertical separation is large,
  the network will never advise a strong turn. \\
   $\phi_8$  &  For a large vertical separation and a previous
 ``weak left'' advisory, the network will either output COC or
 continue advising ``weak left''.\\
   $\phi_9$  & Even if the previous advisory was ``weak right'',
the presence of a nearby intruder will cause the network to output a
 ``strong left'' advisory instead.\\
   $\phi_{10}$  & For a far away intruder, the network advises COC.\\
   \bottomrule
    \end{tabularx}

\end{table*}

\begin{table*}[htbp]
    \centering
    \caption{Formalization of the safety properties defined on ACAS Xu DNNs.}
    \label{tab:formalPro}
    \setlength{\tabcolsep}{1mm}{
    \begin{tabular}{cccccccc}
    \toprule
\multirow{2}{*}{\makecell[c]{\textbf{Property}}}   & \multirow{2}{*}{\makecell[c]{$\rho$\\ (ft)}} & \multirow{2}{*}{\makecell[c]{$\theta$\\ (rad)}} & \multirow{2}{*}{\makecell[c]{$\psi$\\ (rad)}} & \multirow{2}{*}{\makecell[c]{$v_{own}$\\ (ft/s)}} & \multirow{2}{*}{\makecell[c]{$v_{int}$\\ (ft/s)}} & \multirow{2}{*}{\makecell[c]{\textbf{Test}\\ \textbf{DNNs}}}   &  \multirow{2}{*}{\makecell[c]{\textbf{Output}}}  \\
  & &  &  &  & & & \\  
  \hline
       $\phi_1$  & $\rho\geq 55948$ & -- & -- &  $\geq 1145$ & $\leq 60$ & All & COC $\le 1500$ \\ 
        $\phi_2$  &  $\geq 55948$ & -- & -- &  $\geq 1145$ & $\leq 60$ & $\mathcal{N}_{x\ge 2,y}$& COC not max \\ 
         $\phi_3$  &  $\in[1500, 1800]$ & $\in[-0.06 ,0.06]$ &  $\geq 3.10$ & $\geq 980$ & $\geq 960$ & $ \mathcal{N}_{1,y\ge 7}$ & COC not min\\ 
          $\phi_4$  & $\in[1500, 1800]$ & $\in[-0.06 ,0.06]$ & 0 & $\geq 1000$ & $\in[700, 800]$&$ \mathcal{N}_{1,y\ge 7}$ & COC not min\\ 
        $\phi_5$  & $\in[250,500]$& $\in[0.2,0.4]$ & $\approx-\pi$ & $\in[100, 400]$ & $\in[0, 400]$& $\mathcal{N}_{1,1}$& SR min\\ 
         $\phi_6$  &$\in[12000,6000]$ &$\in[0.7,\pi]\vee[-\pi,-0.7]$  & $\approx-\pi$  &  $\in[100, 1200]$& $\in[0, 1200]$&$\mathcal{N}_{1,1}$ & COC min\\ 
           $\phi_7$  & $\in[0,60760]$& $\in[-\pi ,\pi]$ &$\in[-\pi ,\pi]$  & $\in[100, 1200]$ & $\in[0, 1200]$ &$\mathcal{N}_{1,9}$ & SL,SR min\\ 
        $\phi_8$  & $\in[0,60760]$ & $\in[-\pi ,-0.75\pi]$  &$\in[-0.1 ,0.1]$  & $\in[600, 1200]$ & $\in[600, 1200]$&$\mathcal{N}_{2,9}$ & WL$\vee$ COC\\ 
         $\phi_9$  & $\in[2000,7000]$ & $\in[-0.4,-0.14]$ & $\approx-\pi$  & $\in[100, 150]$ & $\in[0, 140]$& $\mathcal{N}_{3,3}$& SR min \\ 
          $\phi_{10}$  & $\in[36000,60760]$ & $\in[0.7 ,\pi]$ & $\approx-\pi$  & $\in[900, 1200]$ &$\in[600, 1200]$ & $\mathcal{N}_{4,5}$& COC min\\ 
          \bottomrule
    \end{tabular}}

\end{table*}

\end{document}